%% file: main.tex
\documentclass{article}

\usepackage[final]{main}

\input{math.tex}

\input{stats.tex}

\input{tex.tex}

\usepackage{graphicx}
\usepackage{blindtext}
\usepackage{cite}
\usepackage[utf8]{inputenc} 
\usepackage[T1]{fontenc}    
\usepackage{hyperref}       
\usepackage{url}            
\usepackage{booktabs}       
\usepackage{amsfonts}       
\usepackage{nicefrac}       
\usepackage{microtype}      
\usepackage{xcolor}         
\usepackage{amsmath}
\usepackage{amssymb}
\usepackage{amsthm}
\usepackage{mathtools}
\usepackage{enumitem}

\theoremstyle{plain}
\newtheorem{theorem}{Theorem}[section]

\newtheorem{lemma}[theorem]{Lemma}

\theoremstyle{definition}

\theoremstyle{remark}

\usepackage{microtype}
\usepackage{graphicx}
\usepackage{subfigure}
\usepackage{booktabs}

\title{Clustering by Attention: Leveraging Prior Fitted Transformers for Data Partitioning}

\author{
  Ahmed~Shokry \\
  Applied Innovation Center\\
  Alexandria, Egypt \\
  \texttt{a.shokry@aic.gov.eg} \\
  \And
  Ayman~Khalafallah \\
  Applied Innovation Center \\
  Alexandria, Egypt \\
  \texttt{a.khalafallah@aic.gov.eg} \\
}

\begin{document}

\maketitle

\begin{abstract}
Clustering is a core task in machine learning with wide-ranging applications in data mining and pattern recognition. However, its unsupervised nature makes it inherently challenging. Many existing clustering algorithms suffer from critical limitations: they often require careful parameter tuning, exhibit high computational complexity, lack interpretability, or yield suboptimal accuracy, especially when applied to large-scale datasets.
In this paper, we introduce a novel clustering approach based on meta-learning. Our approach eliminates the need for parameter optimization while achieving accuracy that outperforms state-of-the-art clustering techniques. The proposed technique leverages a \textit{few} pre-clustered samples to guide the clustering process for the entire dataset in a single forward pass. Specifically, we employ a pre-trained Prior-Data Fitted Transformer Network (PFN) to perform clustering. The algorithm computes attention between the pre-clustered samples and the unclustered samples, allowing it to infer cluster assignments for the entire dataset based on the learned relation. We theoretically and empirically demonstrate that, given just a few pre-clustered examples, the model can generalize to accurately cluster the rest of the dataset.
Experiments on challenging benchmark datasets show that our approach can successfully cluster well-separated data without any pre-clustered samples, and significantly improves performance when a few clustered samples are provided.
We show that our approach is superior to the state-of-the-art techniques. These results highlight the effectiveness and scalability of our approach, positioning it as a promising alternative to existing clustering techniques.
\end{abstract}

\section{Introduction}
\label{intro}

Data clustering refers to the task of discovering inherent groupings or structures within high-dimensional datasets, based on a predefined similarity or distance metric. The effectiveness of clustering methods is highly sensitive to several factors, including the representation of the input data, the dimensionality of the feature space, and the level of noise present in the dataset~\citet{boja2011clusters}. These factors often lead to significant variations in clustering performance across different applications and domains.

Broadly, clustering algorithms are classified into two principal categories: hierarchical and partitional methods~\citet{kantardzic2011data, karaboga2011novel}. Hierarchical clustering constructs a nested tree of clusters using either a bottom-up (agglomerative) or top-down (divisive) approach. In agglomerative hierarchical clustering, each data point is initially treated as an individual cluster, and pairs of clusters are successively merged based on a linkage criterion. In contrast, divisive hierarchical clustering begins with the entire dataset as a single cluster and recursively partitions it into smaller clusters~\citet{jain1999data, rokach2010data}.

On the other hand, partitional clustering aims to partition the dataset into a set of mutually exclusive clusters, such that each data point belongs to exactly one cluster. Among partitional methods, prototype-based clustering is particularly prominent. In this approach, each cluster is represented by a prototype—typically the mean or median of its constituent data points—and the algorithm seeks to minimize an objective function that quantifies the overall distance between each data point and the prototype of its assigned cluster~\citet{mirkin1996mathematical}. A widely adopted example of this class is the K-means algorithm, which is favored for its simplicity and computational efficiency~\citet{macqueen1967some}.

Prior Fitted Networks (PFNs) are a recently developed class of models that leverage the expressive power of Transformer architectures~\citet{vaswani2017attention} to perform Bayesian inference in a highly efficient and scalable manner. The core idea behind PFNs is to approximate the posterior predictive distribution (PPD) using a neural network trained entirely on synthetic data sampled from a known prior distribution—without requiring real-world training data or iterative optimization during inference. This enables PFNs to generate predictions in a single forward pass, making them both fast and computationally efficient.
The concept of PFNs was first introduced by~\citet{muller2021transformers} for general-purpose Bayesian inference. The model is trained offline using synthetic datasets generated from a probabilistic prior and corresponding likelihood functions. The training objective is to approximate the conditional distribution of the target variable given a context (i.e., a set of input-output pairs) and a query point. By learning from a large space of tasks sampled from the prior, the PFN generalizes to new tasks drawn from the same or related priors, effectively performing amortized Bayesian inference.
Subsequent works have extended the PFN framework to various domains. In classification, ~\citet{hollmann2022tabpfn} and ~\citet{hollmann2025accurate} proposed TabPFN, which applies PFNs to tabular data for supervised classification tasks, demonstrating strong performance on few-shot learning problems. In the context of meta-learning and model evaluation, ~\citet{adriaensen2022efficient} employed PFNs to extrapolate learning curves, providing early predictions of a model’s future performance based on limited training progress. Furthermore, in time-series forecasting, ~\citet{khurana2023forecastpfn} introduced ForecastPFN, a PFN-based method that achieves competitive accuracy while maintaining the efficiency of a single forward-pass prediction.
The key advantage of PFNs lies in their ability to perform inference without fine-tuning or explicit model selection at test time. By training the network entirely on artificial data sampled from a prior, PFNs bypass the need for task-specific labeled data and hyperparameter tuning, while still providing well-calibrated uncertainty estimates through their approximation of the posterior predictive distribution. This makes PFNs a compelling framework for various applications requiring fast, reliable inference with minimal supervision.

In this paper, we explore the use of Prior Fitted Networks (PFNs) for data clustering, introducing an algorithm that combines the efficiency of Transformer-based inference with the flexibility of meta-learning. Unlike traditional clustering methods, which often rely on iterative optimization, careful parameter tuning, or assumptions about data distributions, our method performs clustering in a single forward pass of a pre-trained PFN Transformer. The key idea is to provide a few pre-clustered samples from the dataset as input tokens, alongside unclustered data points. The Transformer then calculates attention between the pre-clustered and unclustered samples, effectively propagating cluster information through the attention mechanism. As a result, the model predicts the cluster assignments for all unclustered points without requiring any retraining or fine-tuning. This approach differs significantly from previous PFN applications, which have focused on supervised tasks such as classification~\citet{hollmann2022tabpfn}, forecasting~\citet{khurana2023forecastpfn}, and general Bayesian inference~\citet{muller2021transformers}, none of which address unsupervised learning scenarios. Moreover, unlike existing clustering techniques such as K-means~\citet{macqueen1967some}, hierarchical clustering, or density-based methods, our algorithm does not require iterative refinement or hyperparameter selection and can generalize from minimal supervision. We provide both theoretical and empirical evidence that the model can accurately cluster an entire dataset using only a few pre-clustered examples, demonstrating strong performance across benchmark datasets. This work not only introduces a new direction for clustering via PFNs but also showcases the broader potential of amortized, attention-based inference for unsupervised learning tasks.

%
We evaluate the performance of the proposed algorithm on different challenging datasets. The results show that the algorithm can cluster easily separable datasets without seeing any clustered samples from the data. If a few clustered samples are available, the algorithm can use them efficiently to improve clustering accuracy. Results also show that the algorithm can achieve state-of-the-art accuracy with a running time that is comparable to classical clustering algorithms.

The remaining sections are organized as follows: Section~\ref{sec:methodology} explains the details of the clustering algorithm. Section~\ref{sec:results} presents the results. Section~\ref{sec:discussion} and \ref{sec:conclusion} discuss the proposed algorithm and its limitations and conclude the work, respectively.

\section{Clustering Algorithm}
\label{sec:methodology}
We start the section by explaining the use of PFN for clustering then we explain the mathematical details of the clustering algorithm.

\subsection{Prior Data Fitted Networks for Clustering}
\label{subsec:pfn}
\citet{muller2021transformers} introduced prior-data fitted networks, which are Transformers used for approximate Bayesian prediction in supervised learning. These networks incorporate a prior that establishes a set of hypotheses $\mathcal{H}$ governing the relationship between input data $x$ and output $y$. Each hypothesis $h\in\mathcal{H}$ corresponds to a specific data distribution, and we can generate training datasets $\Dcal_{train} := \{(x_i,y_i)\}_{i=1}^{i=N}$ by sampling from these hypotheses.

The PPD for a test sample $x_{test}$ represents the distribution of possible outcomes, denoted as $p(\cdot|x_{test},\Dcal_{train})$, which depends on the training dataset $\Dcal_{train}$. In order to obtain the PPD, it is necessary to perform integration over all hypotheses $\mathcal{H}$.

\begin{equation}
    p(y|x,\Dcal) \propto \int_{\mathcal{H}} p(y|x,h) p(\Dcal|h) p(h) dh.
    \label{eq:ppd}
\end{equation}

The prior probability $p(h)$ term acts as a weight which determines the significance of the hypothesis $h\in\mathcal{H}$ and the likelihood $p(\Dcal|h)$ of the data $\Dcal$ given $h$.

The prior sampling is defined as follows,
\begin{equation}
 p(D)=\E_{h \sim p(h)}[p(\Dcal|h)]   
\end{equation}

which samples hypotheses with $h \sim p(h)$ before generating synthetic datasets with $\Dcal \sim p(\Dcal|h)$.

In order to train the PFN to predict the outcomes of $\Dcal_{test} \subset \Dcal$, synthetic datasets are drawn ($\Dcal_{train} = \Dcal \setminus D_{test}$) to optimize the PFN parameters $\theta$.
The loss of the PFN training is the cross-entropy on synthetic dataset.

For a single pair $(x_{test},y_{test}) \in \Dcal_{test}$, the training loss can be written as
\begin{equation}
    \mathcal{L} = \E_{(\{(x_{test}, y_{test})\} \cup \Dcal_{train}) \sim p(\Dcal)} [-\log q_{\theta}(y_{test}|x_{test},\Dcal_{train})].
    \label{eq:lpfn}
\end{equation}

As shown by \citet{muller2021transformers}, minimizing this loss approximates the true Bayesian posterior predictive distribution. Note that the training phase is performed once offline.

We use the pre-trained PFN for clustering $D_n$ samples.  Specifically, we used PFN to predict the cluster $c \in \Ccal$ where the sample $x \in X$ belongs given a set of few pre-clustered examples $D_k$. The clustered samples $D_k$ act as a sequence of tokens of length $k$ to the Transformer. The Transformer then calculates the attention between the pre-clustered samples $D_k$. The unclustered samples $D_n \setminus D_k$ then attend to the pre-clustered samples which yields predictions of the cluster numbers for the unclustered samples in a single forward pass.

\subsection{Mathematical Model}
Consider a clustering problem with input $D_n = (x_{i})_{i = 1}^{n}$, where  $x \in \mathbb{R}^d$ (with $d$ features). Given $D_k = (x_{i}, c_{i})_{i = 1}^{k}$ clustered samples, where $k<<n$, our goal is to get the the cluster $c \in \Ccal$ for each sample $x \in D_n$. Hence, we need to predict the conditional probabilities $\Pcal(c \mid \bx)$.

We treat the pre-trained PFN $q_{\wh{\btheta}}(c \mid \bx, \cdot)$ as an untrained predictor for $\Pcal(c \mid \bx)$. 
The estimation error can be represented in terms of bias and variance components \citet{nagler2023statistical}:

\begin{align*}
	 &\quad \quad \; q_{\btheta}(c \mid \bx, D_k)  - \Pcal(c \mid \bx) \\
     &=\quad  \underbrace{q_{\btheta}(c \mid \bx, D_k) - \E_{D_k \sim \Pcal^k}[q_{\btheta}(c \mid \bx, D_k)]}_{\text{variance}}  
     &\quad +\; \underbrace{\E_{D_k \sim \Pcal^k}[q_{\btheta}(c \mid \bx, D_k)]   - \Pcal(c \mid \bx)}_{\text{bias}}.
\end{align*}

\begin{lemma}
The proposed model can learn how to cluster given a few clustered samples, $D_k$.
\end{lemma}

\begin{proof}

When we reveal more clustered samples (i.e. $k$ increases) from $\Dcal_n$, the influence of individual set of samples decreases which bound the variance of the PFN model $q_{\btheta}$. Similar to~\citet{nagler2023statistical}, suppose there are $\alpha > 0$ and $L < \infty$, such that for large enough $k$ and almost all data sets $\Dcal_k, \Dcal_k'$ differing only in one sample,
\begin{align} 
         \bigl| q_{\btheta}\bigl(c \mid \bx, \Dcal_k\bigr) - q_{\btheta}(c \mid \bx, \Dcal_k') \bigr|  
         \le L k^{-\alpha}. \label{eq:bdiff}
\end{align} 
 
Using \citep{mcdiarmid1989method},
\begin{align*}
\Pr\bigl(	\bigl| q_{\btheta}(c \mid \bx, \Dcal_k) - \E[q_{\btheta}(c \mid \bx, \Dcal_k)] \bigr| > \epsilon\bigr) \le 2\exp\biggl(-\frac{2 \epsilon^2}{L^2 k^{1 - 2\alpha}}\biggr).
\end{align*}

This implies, 
\begin{align*}
	\sum_{k = 1}^\infty \Pr\bigl(	\bigl| q_{\btheta}(c \mid \bx, \Dcal_k) - \E[q_{\btheta}(c \mid \bx, \Dcal_k)] \bigr| > \epsilon\bigr) \le 2 	\sum_{k = 1}^\infty \exp\biggl(-\frac{2 \epsilon^2}{L^2 k^{1 - 2\alpha}}\biggr) < \infty.
\end{align*}

The Borel-Cantelli lemma then implies that, almost surely,
\begin{align*}
	\lim_{k \to \infty} \bigl| q_{\btheta}(c \mid \bx, \Dcal_k) - \E[q_{\btheta}(c \mid \bx, \Dcal_k)] \bigr| \le  \epsilon
\end{align*}
Since $\epsilon$ is arbitrary,
\begin{align*}
	  \lim_{k \to \infty} q_{\btheta}(c \mid \bx, \Dcal_k) - \E[q_{\btheta}(c \mid \bx, \Dcal_k)]  = 0 \quad \text{with high probability}.
	\end{align*}
This means that as more we reveal clustered samples, as more the variance decreases which explains how the proposed model with fixed parameters $\theta$ can learn how to cluster given a few clustered samples, $D_k$, independently from the model parameters $\theta$.
 \end{proof} 

\section{Experiments}
\label{sec:results}
We start by showing the performance of the proposed algorithm for small challenging datasets. Then, we move to a large real dataset.

\subsection{Small challenging dataset}
We compare the proposed algorithm to some popular clustering algorithms on different small challenging datasets ($n=2000$). The parameters of each of these algorithms have been tuned to produce good clustering results for all the classical clustering algorithms~\citet{scikit-learn}. Figure~\ref{fig:challenge_cluster} shows the performance of the different techniques in terms of the running time $T$ and the $V$-measure~\citet{rosenberg2007v}. The figure shows that the proposed technique can cluster easily separable data, e.g. the last row, without seeing any pre-clustered samples ($k=0$). For a medium complexity for sample distribution, e.g. the third row, the algorithm achieves comparable accuracy also without seeing any clustered samples. For challenging complex datasets, the algorithm needs to see some pre-clustered samples. We show the results of the proposed clustering algorithm for $k=100$. The results show that the clustering algorithm can achieve a superior accuracy than the classical techniques. The running time~\footnote{We run the experiments on GPU. We investigate the use of CPU in the next subsection. } for the proposed algorithm is comparable to other algorithms.    
\begin{figure*}
      \centering
   \includegraphics[scale=0.27]{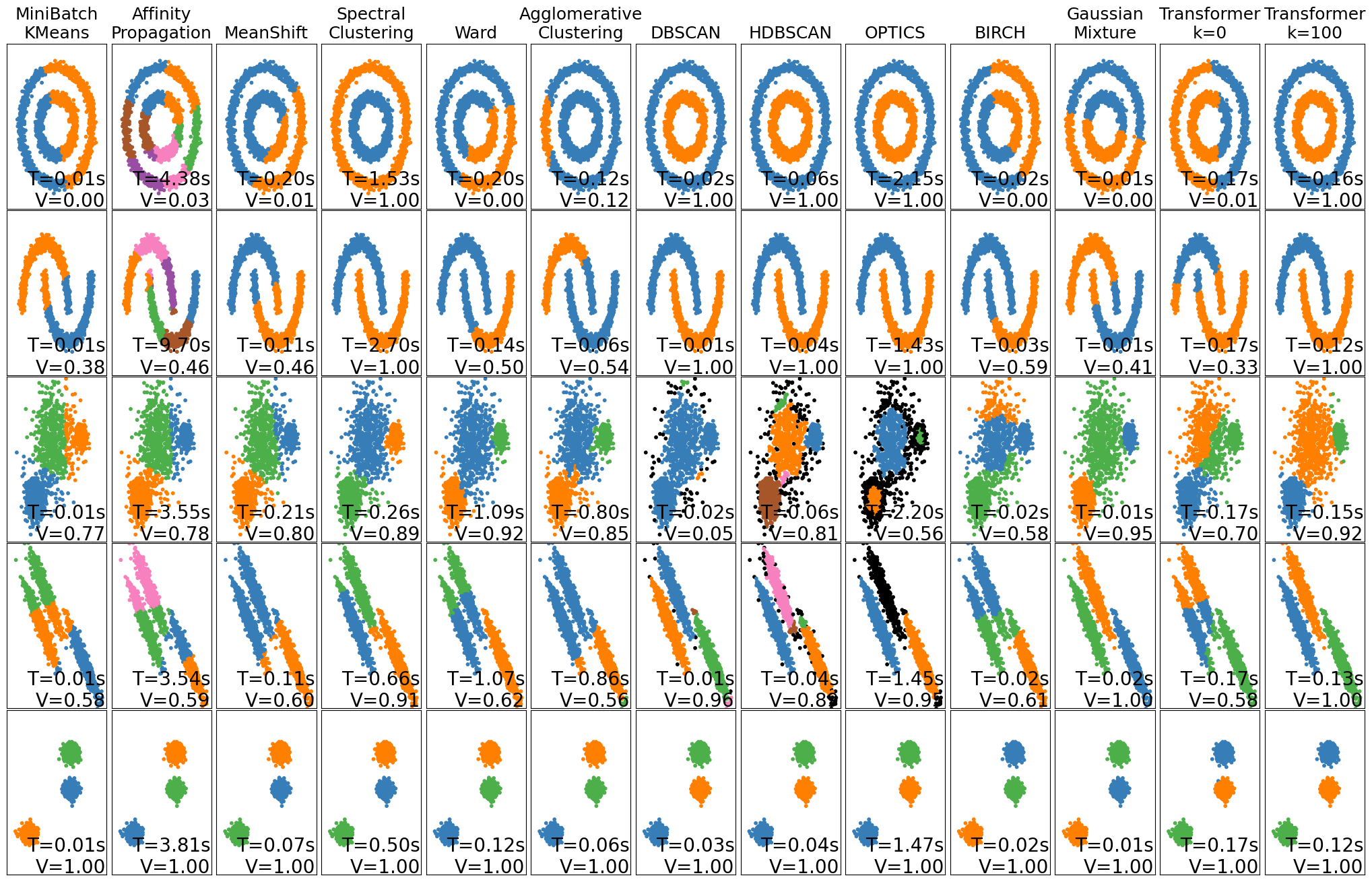}
  \caption{Clustering results for small challenging datasets.} \label{fig:challenge_cluster}
\end{figure*}

\subsection{Large Real Dataset}
We use MNIST dataset to evaluate the performance of the different clustering algorithms~\citet{deng2012mnist}. We start by showing the accuracy of the proposed clustering technique with the increase of the number of pre-clustered samples ($k$). Then, we present a comparison between the proposed clustering algorithm and the classical clustering techniques in terms of accuracy and running time. 

\paragraph{Pre-clustered Samples ($k$):} Figure~\ref{fig:k_effect}, shows the effect of increasing the number of pre-clustered samples seen by the algorithm on the algorithm's performance. As expected, increasing the number of clustered samples increases algorithm performance as it enables the model to know how efficiently can separate the samples.  

\begin{figure*}[!t]
      \centering
   \includegraphics[width=10cm,height=7cm]{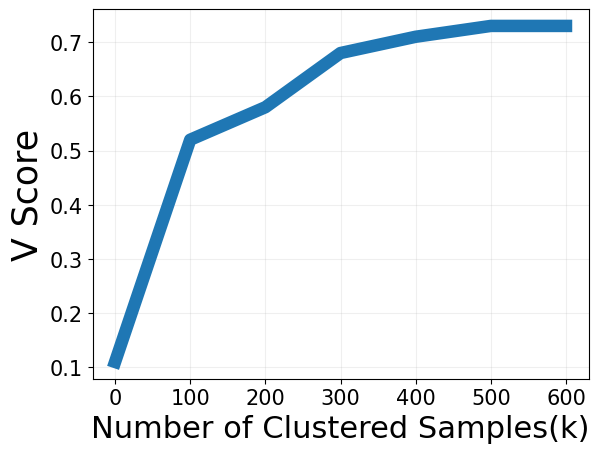}
  \caption{Effect of increasing the number of pre-clustered samples on the algorithm performance.}\label{fig:k_effect}
\end{figure*}

\paragraph{Algorithm Performance:} 
Figure~\ref{fig:performance}, compares the proposed algorithm performance against the classical techniques. The figure shows that our algorithm achieves better accuracy than the classical clustering techniques. This accuracy can be further enhanced by using more pre-clustered samples if available.  

\begin{figure*}[!t]
      \centering
   \includegraphics[scale=0.3]{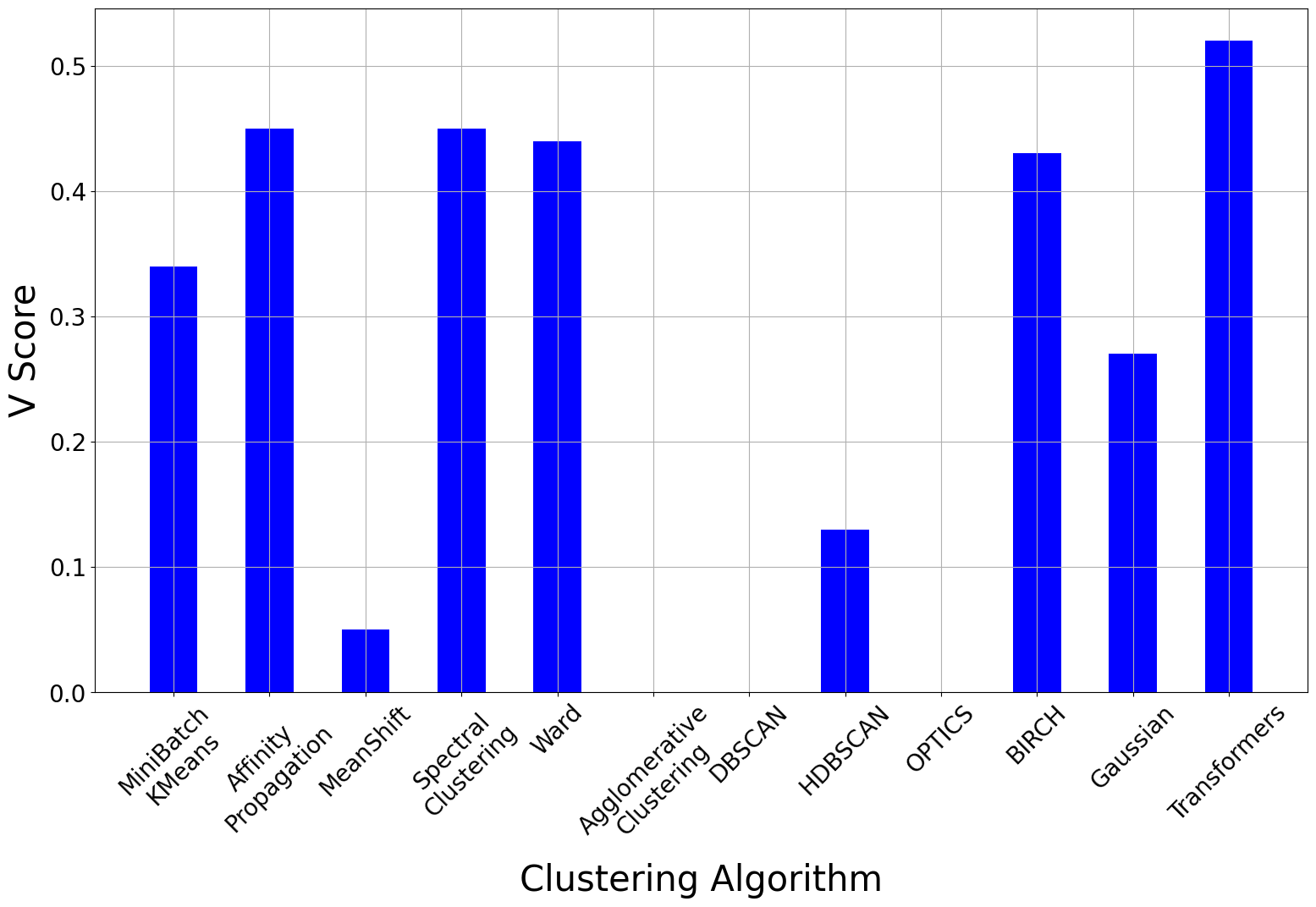}
  \caption{Performance of the different clustering algorithms.}\label{fig:performance}
\end{figure*}

\paragraph{Running Time:} 
Figure~\ref{fig:time}, shows the running time comparison for the different clustering algorithms. The figure shows that the implementation of the algorithm on GPU is among the fastest running times. Moreover, the CPU implementation is still comparable with the other algorithms. 

\begin{figure*}[!t]
      \centering
   \includegraphics[scale=0.3]{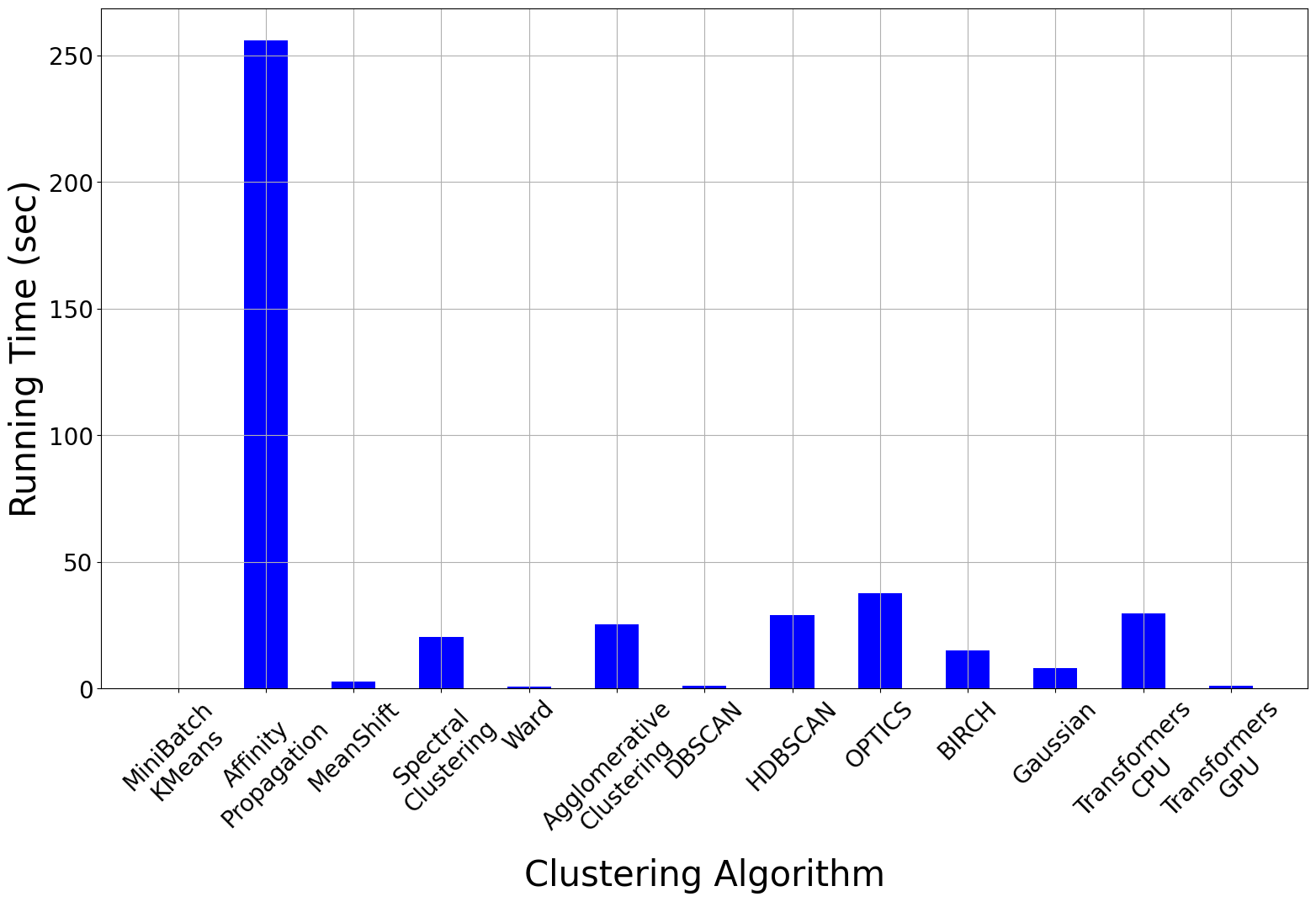}
  \caption{Running time of the different clustering algorithms.}\label{fig:time}
\end{figure*}

\section{Discussion}
\label{sec:discussion}
The experimental results presented in the previous section demonstrate that the proposed clustering algorithm is capable of accurately and efficiently separating data samples into meaningful groups. For datasets with simple or well-separated distributions, the algorithm can successfully perform clustering without the need for any pre-clustered examples. This highlights the model's ability to leverage the structure of the data directly through its learned prior and self-attention mechanism. In contrast, for datasets with more complex or overlapping distributions, the algorithm benefits from being provided with a few pre-clustered samples. These examples act as informative anchors that guide the attention mechanism within the Prior Fitted Transformer, allowing the model to propagate cluster membership information and improve clustering performance. The number of pre-clustered examples required varies with the complexity of the dataset: more intricate distributions typically require a larger number of labeled examples to achieve optimal separation. However, even with a small number of examples, the algorithm demonstrates strong generalization capabilities and steadily improves as more labeled examples are introduced.

In terms of performance, the proposed method consistently outperforms widely-used clustering algorithms such as K-means and hierarchical clustering, particularly in settings where minimal supervision is available. Importantly, it achieves this superior accuracy while maintaining a comparable computational runtime. Due to its fast inference and strong generalization ability, the algorithm can be effectively used as a foundational module in a variety of machine learning pipelines, including tasks such as data augmentation, semi-supervised learning, and dimensionality reduction.

Despite its advantages, the proposed algorithm also inherits the limitations of the Transformers architecture. Specifically, the attention mechanism introduces a space and time complexity of 
$O(n^2)$. This quadratic scaling can become a bottleneck for very large datasets. To address this, several approaches have been proposed in the literature to reduce the computational overhead of attention mechanisms. For example, FlashAttention~\citet{dao2022flashattention} offers a memory-efficient implementation of scaled dot-product attention; Longformer~\citet{beltagy2020longformer} introduces sparse attention patterns to reduce complexity; and BigBird~\citet{zaheer2020big} combines random, global, and local attention to achieve linear complexity. Integrating these scalable attention mechanisms into the PFN framework represents a promising direction for future work to further enhance the scalability of the proposed clustering approach.

\section{Conclusion}
\label{sec:conclusion}
In this work, we introduced a robust and accurate clustering algorithm that eliminates the need for parameter tuning while achieving performance that exceeds state-of-the-art clustering methods. Our approach leverages a pre-trained Prior Fitted Transformer Network (PFN), which is capable of learning from a small number of pre-clustered samples and inferring the cluster assignments for the entire dataset in a single forward pass. We provided theoretical justification showing that the algorithm can generalize to cluster any arbitrary dataset. Experimental evaluations on a variety of challenging benchmark datasets demonstrated that the proposed method effectively separates samples with high accuracy, while maintaining a runtime comparable to traditional clustering techniques.

\medskip
\bibliography{main}
\bibliographystyle{unsrtnat}


\end{document}

%% file: math.tex
\usepackage{bm}
\usepackage{dsfont}

\newcommand{\E}{\mathds{E}}

\newcommand{\bx}{\bm{x}}

\newcommand{\Ccal}{\mathcal{C}}
\newcommand{\Dcal}{\mathcal{D}}

\newcommand{\Pcal}{\mathcal{P}}

\newcommand{\btheta}{\bm{\theta}}



%% file: stats.tex
\providecommand{\Pr}{}
\renewcommand{\Pr}{\mathbb{P}}

\newcommand{\wh}[1]{\widehat{#1}}

%% file: tex.tex
\makeatletter
\def\blfootnote{\gdef\@thefnmark{}\@footnotetext}
\makeatother